\documentclass{article}

\usepackage{arxiv}

\usepackage[utf8]{inputenc} 
\usepackage[T1]{fontenc}    
\usepackage{hyperref}       
\usepackage{url}            
\usepackage{booktabs}       
\usepackage{amsfonts}       
\usepackage{nicefrac}       
\usepackage{microtype}      
\usepackage{lipsum}		
\usepackage{graphicx}
\usepackage{natbib}
\usepackage{doi}

\usepackage{amsmath}
\usepackage{hyperref}
\usepackage{url}
\usepackage{algorithm}
\usepackage{algpseudocode}
\usepackage{amsthm}
\newtheorem{theorem}{Theorem}
\newtheorem{lemma}{Lemma}
\usepackage{booktabs}
\usepackage{multirow}
\usepackage{xcolor}
\usepackage{graphicx}
\usepackage{tcolorbox}
\tcbuselibrary{breakable}
\newtcolorbox{promptbox}[1]{
  breakable, colback=gray!5, colframe=gray!60,
  coltitle=white, colbacktitle=gray!60,
  title=#1, fonttitle=\bfseries\small,
  fontupper=\ttfamily\small
}

\title{Online Verification of Language Model Responses Under Cost Constraints}

\author{ {Erfan Hajihashemi} \\
	Department of Electrical Engineering \& Computer Science\\
	University of California, Irvine\\
	\texttt{ehajihas@uci.edu} \\
	\And
{Yanning Shen} \\
	Department of Electrical Engineering \& Computer Science
\\
	University of California, Irvine \\
	\texttt{yannings@uci.edu} \\
}

\hypersetup{
pdftitle={A template for the arxiv style},
pdfsubject={q-bio.NC, q-bio.QM},
pdfauthor={David S.~Hippocampus, Elias D.~Striatum},
pdfkeywords={First keyword, Second keyword, More},
}

\begin{document}

\maketitle

\begin{abstract}
As large language models are increasingly deployed for multi-step reasoning, verifying the correctness of their outputs has become essential for maintaining reliability at scale. Verifying the correctness of large language model outputs is often done by querying a costly ground-truth oracle, which is impractical to invoke at every step in an online setting. Prior work addresses this by querying a single weak verifier on every step, and using its score to decide whether the costly strong verifier needs to be queried as well, reserving strong verification for only a small fraction of the steps. However, a single fixed weak verifier may not perform consistently well as the subject matter or difficulty of incoming queries changes over time, and committing to one in advance risks either overly costly or inaccurate verification. We introduce OMVV (Online Multi-Verifier Verification), an algorithm that maintains a pool of $K$ candidate weak verifiers with differing cost and verification performance, and adaptively routes each round's decision to a verifier selected via an online score combiner and an exponential-weights routing policy. OMVV provides a distribution-free, finite-time guarantee on false-accept and false-reject rates across the full pool of verifiers, and further achieves sublinear regret against the best fixed verifier in hindsight under a combined cost and consistency objective. Experiments on reasoning dataset benchmarks show that OMVV achieves higher accuracy at lower verification cost than any single fixed verifier, across a range of operating budgets.
\end{abstract}
\section{INTRODUCTION}
Large language models increasingly solve complex problems by producing multi-step chain of thought reasoning traces \citep{wei2022chain}. As such traces are used to drive downstream decisions, search, and agentic pipelines, an incorrect final answer that goes unflagged can propagate silently into the final output. Verifying a candidate solution against ground-truth is an effective remedy: training a separate model to judge the correctness of a candidate answer and using it to rerank or filter samples substantially improves accuracy \citep{cobbe2021training}.

In many practical applications, a language model does not answer queries in isolation: it faces a stream of queries that arrive sequentially over time, which can naturally be modeled as an online environment. In such a stream, verification is not only a per query concern, since a mistake left unflagged at one step can propagate forward into the queries and responses that follow, whether by feeding a wrong intermediate result into subsequent reasoning or by shaping later interactions around a flawed answer. This has motivated a growing line of work that studies verification of language models directly as an online problem \citep{kiyani2026when, balcan2026online}. However, providing an oracle that checks the correctness of every single generated response in this online setting is often not doable in practice.

To avoid the cost of obtaining ground-truth for every query, a weak verifier, an inexpensive way to estimate correctness that is cheaper than the oracle but also less reliable, can be integrated. Relying on such a verifier makes it possible to decide most responses without ever querying the oracle, calling the oracle only on the subset of responses the weak verifier is least certain about. However, a single weak verifier might not be sufficient throughout the entire stream: how well it matches the queries it encounters depends on their subject and difficulty, both of which can shift over the course of a stream. A verifier well suited to easier queries may become overconfident as difficulty increases, while a verifier calibrated for harder queries may query the oracle unnecessarily on queries it could have resolved on its own. We propose OMVV, which addresses this by maintaining a pool of candidate weak verifiers with differing cost and verification performance rather than committing to a single one. At each round, OMVV routes the response to one of the candidate verifiers. The oracle is queried only when the selected verifier is itself uncertain about the response. This routing decision is learned online, and OMVV additionally satisfies a regret guarantee on it, ensuring that its accuracy and cost approach those of the best fixed verifier in the pool in hindsight. Across mathematical reasoning, logic, and factual tasks, OMVV outperforms single verifier selective verification \citep{kiyani2026when}, run separately with each of the two candidate weak verifiers, on accuracy and cost simultaneously.

\textbf{Contributions.} Overall, our contributions can be summarized as follows:
\begin{enumerate}
\item[I)] We introduce OMVV, an online algorithm for verifying language model outputs that maintains a pool of weak verifiers with differing cost and verification performance and adaptively routes each response to a candidate verifier based on its historical performance.
\item[II)]It is proven that OMVV maintains distribution-free finite-time control over both wrongly accepting an incorrect
response and wrongly rejecting a correct one. We further establish a regret guarantee on the routing decision itself, showing that OMVV's accuracy and cost approach those of the best fixed verifier in the pool in hindsight.
\item[III)] Through experiments on mathematical reasoning, logic, and factual tasks, it is demonstrated that OMVV outperforms single-verifier selective verification \citep{kiyani2026when} achieving better accuracy and cost simultaneously.
\end{enumerate}
\section{RELATED WORKS}
Verifying language model outputs without exhaustive ground-truth checking has been approached from several directions. One line of work improves individual verification mechanisms. Some use LLM as a judge to evaluate outputs directly \citep{liu-etal-2023-g}, while others check generated claims against source documents \citep{tang-etal-2024-minicheck}. Other work trains  verifiers via next token prediction rather than discriminative classification \citep{zhang2025generative}, or scores intermediate reasoning steps with process level reward models \citep{wang-etal-2024-math}. A separate line of work improves test time response selection by combining multiple verifiers. Weaver aggregates weak verifiers using weak supervision to estimate verifier accuracies with minimal labeled data \citep{saad-falcon2025shrinking}, while Multi-Agent Verification combines best of $n$ sampling with equal weight voting across multiple aspect verifiers to scale test time compute \citep{lifshitz2025multiagent}. Both use a fixed verifier aggregation procedure at inference time: candidate responses are generated and then scored by an ensemble of verifiers, rather than adaptively deciding, for each query, which verifier to use or whether to escalate to a costly oracle. Closest to our setting is Selective Strong Verification (SSV) \citep{kiyani2026when}, which uses a single weak verifier together with adaptive accept and reject thresholds. In contrast, OMVV considers a pool of $K$ candidate weak verifiers and learns how to route each round among them through an adaptive policy, while also establishing a regret guarantee for the routing decisions. SSV's accept and reject framework is closely related to the broader selective prediction literature. Foundational work poses classification with a reject option as risk minimization with an explicit rejection cost \citep{JMLR:v9:bartlett08a} and formalizes the resulting risk coverage tradeoff \citep{JMLR:v11:el-yaniv10a}; more recent work extends selective classification to online settings with partial feedback \citep{gangrade2021online}. A related line of work is learning to defer. The goal is to learn, for each instance, whether to route the decision to a model or to a human expert \citep{mozannar2020consistent, mozannar2023should, okati2021differentiable, verma2022calibrated}. In contrast, we consider online routing among a heterogeneous pool of weak verifiers with differing costs and verification performance, together with the option to escalate to a costly oracle.
\section{PROBLEM SETUP}
\label{sec:psetup}
The verification of language models in an online setting is studied in this work. Queries arrive sequentially, and multiple candidate verifiers are available to verify the generated responses. There is also an oracle that provides the ground-truth verification label for a pair when queried, at a cost higher than that of any candidate verifier. The generator receives queries sequentially and, for each query, produces a fixed set of $M$ candidate responses under a predefined generation budget, consistent with prior selective-verification settings that consider multiple candidate generations per query \citep{kiyani2026when}; these candidates are then given to the verifier one at a time for verification.

\textbf{Online query and response:} In the online setting, at each time $t = 1, \dots, T$, a query response pair $p_t = (x_t, r_t)$ is presented for verification, where $r_t$ is one candidate response to the query $x_t$. Next, the pair will be verified to see whether the generated response $r_t$ is a valid response for $x_t$ or not. We make no assumption that the incoming query distribution remains fixed over time; the subject matter or difficulty of the queries may change throughout the stream.

\textbf{Oracle verification:} There is an oracle that provides the ground-truth verification for a given pair $p_t$ if needed. The oracle verifier maps $p_t$ to a binary value, denoted $g_t \in \{0,1\}$, which indicates whether the generated response $r_t$ is correct or not.

\textbf{Multi-verification:}
A set of $K$ candidate verifiers is available at each time $t$. Based on the historical performance of the verifiers, a single verifier $k_t$ is selected among the candidate verifiers at each round. The verifier maps $p_t$ into a discrete space $\{A, R, Q\}$, and we denote the resulting decision as $d_t \in \{A, R, Q\}$, where $A$ stands for accepting the response $r_t$, $R$ stands for rejecting the generated response for query $x_t$, and $Q$ stands for querying the oracle, which models the uncertainty of the verifier in verifying the given pair $p_t$.

\textbf{Problem Statement.} At any time step $0 < t \leq T$:
\begin{center}
\fbox{\begin{minipage}{0.95\textwidth}
\vspace{2pt}
\textbf{Given:}
\begin{itemize}
    \item a query response pair $p_t = (x_t, r_t)$;
    \item $K$ verifiers 
    with different verification performance and cost;
    \item an oracle that provides ground-truth verification $g_t \in \{0,1\}$ for a pair $p_t$ when queried;
    \item target error rates $\alpha_I, \alpha_{II} \in (0,1)$ bounding the false-accept and false-reject rates;
    \item the historical pairs $(x_\tau, r_\tau)_{\tau=1}^{t-1}$ and their verification outcomes.
\end{itemize}
\textbf{Goal:} adaptively select a verifier for each pair $p_t$ to achieve high verification accuracy at low cost, while keeping the false-accept and false-reject rates below $\alpha_I$ and $\alpha_{II}$, respectively.
\end{minipage}}
\end{center}
\section{METHODOLOGY}
To verify the generated answer of a language model, a possible way is to query an oracle that provides ground truth on whether the generated response for the given query is correct \citep{cobbe2021training, lightman2024lets}. In the online setting, however, querying the oracle at every timestep is prohibitively expensive, which motivates deciding most rounds with a cheaper weak verifier and reserving the oracle for the rounds the verifier is least certain about \citep{kiyani2026when}. Committing to a single such verifier can be limiting, since weak verifiers can differ in their verification performance \citep{saad-falcon2025shrinking}. We therefore maintain a set of $K$ verifiers, all cheaper than the oracle, differing in cost and verification performance, and design an online algorithm that learns which verifier to select at each round from the oracle feedback it collects over time.

Designing such an algorithm requires a way to measure how well a choice of verifier performs at each step, along two complementary axes: the cost of using that verifier, and how closely its decisions agree with the oracle. We formalize both next.

\textbf{Cost.} Each candidate verifier 
incurs a fixed cost $c_{k} \in [0,1]$ to run, reflecting the intuition from before that cheaper verifiers should be preferred when they suffice. Always selecting the cheapest verifier trivially minimizes the cost incurred, but may lead to poorer verification performance, since the cheapest verifier is not necessarily the right one for every input.

\textbf{Error.}  A decision made without the oracle can go wrong in two ways, either accepting an incorrect response or rejecting a correct one. For verifier $k$, let $T_k(T) := \{t \le T : k_t = k\}$ denote the rounds routed to verifier $k$, and define 
\begin{equation}
\label{eq:n0n1}
N_0^{(k)}(T) := \big|\{t \in T_k(T) : g_t = 0\}\big|, \qquad
N_1^{(k)}(T) := \big|\{t \in T_k(T) : g_t = 1\}\big|.
\end{equation}
The per-verifier false-accept and false-reject rates are then
\begin{equation}
\label{eq:erri}
\text{Err}_I^{(k)}(T) := \frac{1}{N_0^{(k)}(T)}
\sum_{t \in T_k(T)} \mathbf{1}\{g_t = 0,\ d_t = A\},
\end{equation}
\begin{equation}
\label{eq:errii}
\text{Err}_{II}^{(k)}(T) := \frac{1}{N_1^{(k)}(T)}
\sum_{t \in T_k(T)} \mathbf{1}\{g_t = 1,\ d_t = R\}.
\end{equation}
The pooled quantities $N_0(T)$, $N_1(T)$, $\text{Err}_I(T)$, $\text{Err}_{II}(T)$ are defined identically, summing over all $t \le T$ without restricting to a single verifier's sub stream. The oracle can always be queried to avoid error entirely, but doing so incurs the maximum possible cost. A good policy must therefore keep cost and error small simultaneously, which is the tradeoff our algorithm is designed to navigate. 
We formalize this tradeoff as the following loss for each candidate verifier $k$ at round $t$:
\begin{equation}
\label{eq:loss}
\ell_t(k) := \lambda \big(s_{k,t} - g_t\big)^2 + (1-\lambda)\, c_k,
\end{equation}
where $\lambda \in [0,1]$ is a tradeoff parameter. In
Equation~\ref{eq:loss}, the first term captures the consistency of verifier $k$ with the oracle: $s_{k,t}$ is the score verifier $k$ would produce for the current pair $p_t$, which determines the decision made by verifier $k$ at time $t$, and a smaller distance from the oracle's label $g_t$ reflects a more reliable decision by
verifier $k$ at this round. We describe how $s_{k,t}$ is computed in further detail later in this section. The second term, $c_k$, is a fixed cost associated with verifier $k$.

We treat choosing which verifier to consult as a $K$-armed decision problem. One verifier is sampled and used at each round, as in the multi-armed bandit setting \citep{auer2002nonstochastic}. Querying every verifier at every round would undermine the cost savings that motivate routing in the first place, so OMVV queries only the sampled verifier $k_t$. The verifier score is therefore observed for $k_t$ alone. This is the defining property of bandit feedback, and it separates our setting from the full information ``prediction with expert advice'' setting \citep{littlestone1994weighted, freund1997decision}, where every expert's loss is revealed every round regardless of which one is followed.

The loss $\ell_t(k)$ in Equation~\ref{eq:loss} has two parts, each observed under a different rule. The cost term $c_k$ is fixed and known in advance for every verifier, so nothing needs to be observed to use it. It is available every round for every $k$, selected or not. The consistency term depends on both the verifier score $s_{k,t}$ and the oracle label $g_t$. Since only the sampled verifier is queried, its score is the only verifier-specific score observed at round $t$, while $g_t$ is observed only when the oracle is queried. The  loss is therefore full information in its cost part, and bandit feedback, thinned further by the oracle query decision, in its consistency part. Feedback that sits between full information and bandit feedback has been studied broadly under partial monitoring \citep{piccolboni2001discrete, cesa2006regret}, and selectively querying the true label is closely related to label efficient prediction \citep{cesa2005minimizing}. Our setting does not need the full generality of partial monitoring, however: it can be handled with the same importance weighted estimators used for standard adversarial bandits \citep{auer2002nonstochastic}, alongside a known, deterministic cost term that needs no observation at all.


\subsection{ALGORITHM}
\label{sec:algorithm}
We now describe OMVV in full: how a verifier is selected at each round, how its threshold pair governs the accept/reject/query decision, and how the sampled verifier's thresholds and score combiner are updated from the oracle's feedback.

In the multi-verification setting, we have $K$ candidate verifiers, each initialized in advance with weight $w_k^1 = 1/K$. The weight $w_k^t$ reflects how favorably verifier $k$ has performed under the objective in Equation~\ref{eq:loss} up to round $t$. Normalizing these weights, $ \bar{w}_k^t := \frac{w_k^t}{\sum_{j=1}^K w_j^t}, $ gives $\bar{w}_k^t$, the probability of selecting verifier $k$ at round $t$. Collecting these normalized weights over all $K$ verifiers gives a probability mass function $\mathbf{W}^t := (\bar{w}_1^t, \dots, \bar{w}_K^t)$. At round $t$, a verifier $k_t\in\{1,\dots,K\}$ is sampled according to $\mathbf{W}^t$ and used to evaluate the incoming pair $p_t=(x_t,r_t)$.

When verifier $k$ is selected at round $t$, it is prompted with the pair $p_t$ and returns a raw score $u_{k,t} \in [0,1]$, with larger values indicating higher confidence that the response is correct. In addition to the raw score, we also use agreement as a second signal of correctness, computed from the generator's own outputs and independent of any verifier's judgment. We combine the two via logistic regression. Formally, the agreement feature $a_t \in [0,1]$ is the fraction of the other candidate responses generated for the same query that agree with the current one. We refer to the result as a combined score $s_{k,t}$, and base each verifier's decision on $s_{k,t}$.
Formally,
\begin{equation}
\label{eq:combiner}
s_{k,t} := \sigma\!\left(\beta_0^{(k),t} + \beta_1^{(k),t} u_{k,t} +
\beta_2^{(k),t} a_t\right)
\end{equation}
where $\sigma$ is the logistic sigmoid and $\beta^{(k),t} := (\beta_0^{(k),t}, \beta_1^{(k),t}, \beta_2^{(k),t})$ is verifier $k$'s coefficient vector, fit separately for each verifier on its own history of oracle queried rounds and refit as more become available. 

Given the combined score $s_{k,t}$, each verifier's decision follows a two-threshold rule: reject below a lower threshold, accept above an upper threshold, and defer to the oracle in between. The intuition behind this structure, as laid out by \citet{kiyani2026when}, is that when a verifier's score is confidently low, the response is likely incorrect and should be rejected; when the score is confidently high, the response is likely correct and should be accepted; and when the score falls in between, neither conclusion is well supported, and the safer choice is to defer to the oracle rather than guess. We adopt this structure for each of our $K$ verifiers, applied to its own combined score $s_{k,t}$.

Each verifier $k$ maintains its own pair of adaptive thresholds $(\tau_R^{(k),t}, \tau_A^{(k),t})$, with $\tau_R^{(k),t} \le \tau_A^{(k),t}$, updated only on rounds when verifier \(k\) is selected and oracle feedback is observed. At round $t$, the combined score $s_{k_t,t}$ of the sampled verifier is compared against that verifier's own thresholds:
\begin{equation}
\label{eq:region}
\text{region}(t) =
\begin{cases}
\text{Accept}, & s_{k_t,t} > \tau_A^{(k_t),t}, \\
\text{Uncertain}, & \tau_R^{(k_t),t} < s_{k_t,t} \le \tau_A^{(k_t),t}, \\
\text{Reject}, & s_{k_t,t} \le \tau_R^{(k_t),t}.
\end{cases}
\end{equation}
A separate threshold pair is kept for each verifier, rather than one shared pair across all $K$, since different verifiers can differ substantially in both the scale of their scores and how those scores relate to correctness.

The action taken at round $t$ follows directly from this region: in the Accept region the response is accepted, in the Reject region it is rejected, and in the Uncertain region, the oracle is queried and its label $g_t$ determines the action. Querying the oracle only in the Uncertain region, however, is not enough on its own. If a verifier's Accept region rarely or never triggers a query, the algorithm does not learn whether the responses it accepts there are actually correct, and an overly permissive threshold could continue accepting incorrect responses without a signal for correction. To prevent this, the algorithm queries the oracle with a small probability even in the Accept and Reject regions, governed by exploration probabilities $P_A$ and $P_R$: with probability $1-P_A$ an Accept-region round is accepted outright, and with probability $P_A$ the oracle is queried to obtain feedback for the
adaptive updates without changing the Accept decision already made; the Reject region is handled symmetrically with $P_R$.

When the oracle is queried, whether from the
Uncertain region or from exploration in the Accept or Reject region, the selected verifier's threshold pair is updated using the resulting label $g_t$. For
verifier $k_t$, with step size $\eta > 0$,
\begin{equation}
\label{eq:tau-update}
\tau_A^{(k_t),t+1} \leftarrow \max\!\left(\tau_R^{(k_t),t},\;
\tau_A^{(k_t),t} + \eta \cdot
\frac{\mathbf{1}[g_t{=}0]\big(\mathbf{1}[s_{k_t,t} > \tau_A^{(k_t),t}]
- \alpha_I\big)}{q_t}\right),
\end{equation}
\begin{equation}
\label{eq:tau-update-r}
\tau_R^{(k_t),t+1} \leftarrow \min\!\left(\tau_A^{(k_t),t+1},\;
\tau_R^{(k_t),t} + \eta \cdot
\frac{\mathbf{1}[g_t{=}1]\big(\alpha_{II} -
\mathbf{1}[s_{k_t,t} \le \tau_R^{(k_t),t}]\big)}{q_t}\right),
\end{equation}
where $q_t \in (0,1]$ is the probability with which the oracle is queried at round $t$ ($P_A$ in the Accept region, $P_R$ in the Reject region, and $1$ in the Uncertain region). Every other verifier's thresholds are left unchanged. Three things about this update are worth making explicit. The $\max(\cdot)$ and $\min(\cdot)$ operations preserve the ordering $\tau_R^{(k_t),t+1} \le \tau_A^{(k_t),t+1}$ after every update, so the region ordering in Equation~\ref{eq:region} is never violated. The accept threshold update only fires when $g_t=0$, since that is exactly the condition under which accepting the response would mean wrongly accepting an incorrect one, and it moves the threshold in the direction that drives the false accept rate toward the target $\alpha_I$; the reject threshold update is the mirror image, tied to $g_t=1$ and the target $\alpha_{II}$. Dividing by $q_t$ 
corrects for partial oracle feedback through importance weighting, so the update is unbiased with respect to the corresponding full feedback update.

The routing weights are updated separately, and every round, for every verifier, not only the sampled one. For each $k$, define the loss estimate
\begin{equation}
\label{eq:loss-estimate}
\hat\ell_t(k) := \lambda \big(s_{k,t} - g_t\big)^2 \cdot
\frac{\mathbf{1}[k_t{=}k]\, O_t}{\bar w_k^t\, q_t} +
(1-\lambda)\, c_k,
\end{equation}
where $O_t \in \{0,1\}$ indicates whether the oracle was queried at round $t$.
The cost term is deterministic and known for every $k$, so it is added every round regardless of which verifier was selected. The consistency term, however, is nonzero only for the sampled verifier on a oracle queried round. Dividing by $\bar w_k^t q_t$ corrects for the fact that the consistency term is observed only when verifier $k$ is sampled and the oracle is queried. Since $\bar w_k^t q_t$ is the probability of that event, inverse probability weighting makes $\hat{\ell}_t(k)$ unbiased for the corresponding full feedback loss. The weight of each verifier is then updated multiplicatively,
\begin{equation}
\label{eq:weight-update}
w_k^{t+1} \leftarrow w_k^t \cdot \exp\!\big(-\eta_H\, \hat\ell_t(k)\big),
\end{equation}
for a fixed learning rate $\eta_H > 0$, and $\bar w_k^{t+1}$ is obtained by renormalizing as before. Algorithm~\ref{alg:omvv} summarizes the full procedure described above.

\begin{algorithm}[t]
\caption{Online Multi-Verifier Verification (OMVV)}
\label{alg:omvv}
\begin{algorithmic}[1]
\Require Target error rates $\alpha_I, \alpha_{II} \in (0,1)$;
tradeoff parameter $\lambda \in [0,1]$; exploration probablities
$\{P_A, P_R\}$; threshold step size $\eta > 0$;
Hedge learning rate $\eta_H > 0$; verifier costs $(c_k)_{k=1}^K$;
\State Initialize weights $w_k^1 \gets 1/K$ for each $k = 1, \dots, K$
\For{$t = 1, 2, \dots$}
  \State Receive pair $p_t = (x_t, r_t)$
  \State \textbf{Route:} normalize weights, $\bar w_k^t \gets
    w_k^t / \sum_j w_j^t$ for each $k$; sample $k_t \sim \bar w^t$
  \State \textbf{Score:} obtain raw score $u_{k_t,t}$; compute combined score $s_{k_t,t}$ (~\ref{eq:combiner})
  \State \textbf{Region test:} set $O_t \gets 0$; compare $s_{k_t,t}$ to
    $(\tau_R^{(k_t),t}, \tau_A^{(k_t),t})$
    \If{$s_{k_t,t} > \tau_A^{(k_t),t}$}
  \State $d_t \gets A$; w.p. $P_A$, query the oracle
    for $g_t$ and set $O_t \leftarrow 1$; set $q_t \gets P_A$
\ElsIf{$s_{k_t,t} \le \tau_R^{(k_t),t}$}
  \State $d_t \gets R$; w.p. $P_R$, query the oracle
    for $g_t$ and set $O_t \leftarrow 1$; set $q_t \gets P_R$
\Else
  \State $d_t \gets Q$; query the oracle for $g_t$ and set $O_t \leftarrow 1$; 
\EndIf
  \State \textbf{Decision:} accept $r_t$ if $d_t=A$ or $(d_t=Q \land g_t=1)$; otherwise reject $r_t$.
  \If{$O_t = 1$}
    \State \textbf{Threshold update} for verifier $k_t$ only
      (~\ref{eq:tau-update}--\ref{eq:tau-update-r})
  \EndIf
  \State (thresholds of every $k \ne k_t$ left unchanged)
  \State \textbf{Routing update:} $\forall k \in [K]$,
    compute the loss estimate $\hat\ell_t(k)$
    (~\ref{eq:loss-estimate}) and update $w_k^{t+1}$
    (~\ref{eq:weight-update})
\EndFor
\end{algorithmic}
\end{algorithm}

\section{THEORETICAL ANALYSIS}
\textbf{Error control.} OMVV updates each verifier's thresholds only on rounds when the oracle is queried and feedback is observed. Since the query stream, the generator, and each verifier's performance are all allowed to be arbitrary and time-varying, it is not obvious that these threshold updates keep the realized false-accept and false-reject rates close to their targets $\alpha_I, \alpha_{II}$. The following theorem shows that they do: OMVV's error rates remain controlled at every finite horizon, with no assumption on the stream, the generator, or the
verifiers.

\begin{theorem}[Error control]
\label{thm:error-control}
Fix a horizon $T \ge 1$. For any $\delta \in (0,1)$, with probability
at least $1 - \delta$,
\[
\text{Err}_I(T) \le \alpha_I + \Delta_{\text{pool}}^I(T, \delta),
\qquad
\text{Err}_{II}(T) \le \alpha_{II} + \Delta_{\text{pool}}^{II}(T, \delta).
\]
Here, $\Delta_{\mathrm{pool}}^I(T,\delta)$ and
$\Delta_{\mathrm{pool}}^{II}(T,\delta)$ are finite sample slack terms that
decrease as the relevant per verifier sample counts increase, for fixed
exploration and algorithm parameters. Their exact form and the proof are
given in Section~\ref{sec:proofs}.
\end{theorem}

\textbf{Regret.} Because OMVV operates online, it does not know in advance which of the $K$ verifiers is the best choice under the combined cost consistency objective in Equation~\ref{eq:loss}. Instead, it must learn an effective routing policy from the feedback observed over time. We therefore evaluate its routing performance through regret relative to the best fixed verifier in hindsight, denoted by $k^\star \in \arg\min_{k\in[K]} \sum_{t=1}^{T}\ell_t(k)$. Formally, let $R_T(k^\star):=\sum_{t=1}^{T}\ell_t(k_t)-\sum_{t=1}^{T}\ell_t(k^\star)$ denote the regret of OMVV's routing decisions relative to $k^\star$.
\begin{theorem}[Routing regret]
\label{thm:regret}
For any horizon $T \ge 1$, OMVV achieves the following expected regret
\[
  \mathbb{E}\big[R_T(k^\star)\big] = O\!\left(\sqrt{T\log K}\cdot
  \sqrt{\frac{\lambda^2 K}{q_{\min}} + \big(1-\lambda^2\big)}\right).
\]
In particular, for fixed $K$, $\lambda$, and $q_{\min}>0$, the expected regret is $O(\sqrt{T})$ and hence sublinear in $T$. $q_{\min}$ is the smallest probability, over the whole horizon, with which the oracle is queried. The proof is given in Section~\ref{sec:proofs}.
\end{theorem}

\begin{figure}[t]
\centering
\includegraphics[width=\textwidth]{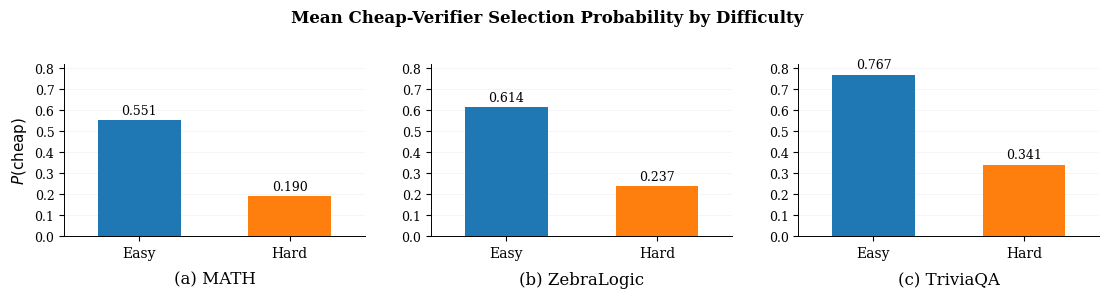}
\caption{Mean probability that OMVV routes to the cheaper verifier (Qwen2-1.5B), by difficulty. On easy problems, OMVV favors the cheaper verifier with substantially higher probability than on hard problems, where it shifts toward the more expensive verifier.}
\label{fig:cheap-verifier-by-difficulty}
\end{figure} 

\section{EMPIRICAL PERFORMANCE}

\textbf{Datasets:} We evaluate OMVV on four tasks spanning mathematical reasoning, logic puzzles, factual question answering, and algebra, to test whether the improvements it provides over a single fixed verifier hold beyond one narrow domain. \textbf{MATH.}  We use 400 problems from the MATH benchmark \citep{hendrycks2021measuring}, which contains five difficulty levels. We select 200 problems from the easiest level, D1, and 200 from the hardest level, D5, as our easy and hard subsets, respectively. \textbf{Zebra Logic.} A set of 200 logic-grid puzzles (100 easy, 100 hard) from ZebraLogic \citep{lin2025zebralogic}, each requiring the solver to infer a unique assignment of attributes to entities satisfying a list of constraints. \textbf{TriviaQA.} A set of 400 factual question-answering examples (200 easy, 200 hard) from TriviaQA \citep{joshi2017triviaqa}. \textbf{DeepMind Mathematics.} A set of 600 algebra problems (300 easy, 300 hard) drawn from the \texttt{linear\_2d} category of the DeepMind Mathematics dataset \citep{saxton2019analysing}, which requires solving a system of two linear equations.

\textbf{Generator and oracle.} For each problem, GPT-4o-mini generates a fixed set of \(M=5\) candidate responses before verification. These candidates are subsequently processed one at a time. Ground truth is provided by an oracle specific to each dataset: for MATH and TriviaQA, the oracle is GPT-4o, acting as an LLM judge that determines whether a candidate response is correct. For Zebra Logic and DeepMind Mathematics, the oracle is a deterministic programmatic grader that checks the generated output against the corresponding ground truth. For the programmatic graders, we assign a normalized unit query cost of $1.0$, matching the oracle cost used for GPT-4o, so that verification cost is reported on a comparable scale across all four datasets.

\textbf{Verifier pool and query stream.} We consider $K=2$ candidate weak verifiers for our experimental analysis: Qwen2-1.5B-Instruct \citep{yang2024qwen2technicalreport} and Qwen2.5-3B-Instruct \citep{qwen2025qwen25technicalreport}, a cheap versus expensive pair with normalized costs of $0.10$ and $0.20$, respectively, relative to an oracle cost of $1.0$. For every dataset, the query stream is constructed as alternating blocks of easy and hard problems, presenting a block of easy problems first, followed by a block of hard problems.

\textbf{Baselines and evaluation metrics.} We compare OMVV against each candidate verifier run independently: Qwen2-1.5B-Instruct and Qwen2.5-3B-Instruct, each following the single-verifier adaptive procedure of \citet{kiyani2026when}. For each run, we report accuracy, the fraction of problems solved correctly; the false-accept and false-reject rates$\mathrm{Err}_I(T)$ and
$\mathrm{Err}_{II}(T)$, defined in Equations ~\ref{eq:erri} and \ref{eq:errii}, along with their average $(\mathrm{Err}_I(T)+\mathrm{Err}_{II}(T))/2$ as a single summary of verification error; cost, the average total verification cost per round, accounting for both the selected weak verifier and any oracle query; and the oracle query frequency, $\mathrm{OracleFreq}(T):=\frac{1}{T}\sum_{t=1}^{T}O_t$, the fraction of rounds on which the oracle is queried. 

\begin{table}[t]
\caption{OMVV versus single-verifier baselines, pooled easy and hard
subsets, at $\alpha_I = \alpha_{II} = 0.2$. $\text{Err}_I$ and
$\text{Err}_{II}$ values exceeding the $0.2$ target are shown in
red; the best value in each of Accuracy, Avg($\text{Err}_I$,
$\text{Err}_{II}$), Cost, and Oracle frequency, within each dataset, is
bolded.}
\label{tab:pooled-results}
\begin{center}
\begin{tabular}{llccccccc}
\toprule
Dataset & Method & Accuracy & $\text{Err}_I$ & $\text{Err}_{II}$ & Avg & Oracle freq & Cost \\
\midrule
\multirow{3}{*}{MATH}
 & Qwen2-1.5B & 0.6775 & \textcolor{red}{0.7655} & 0.0845 & 0.425 & 0.2713 & 0.3713 \\
 & Qwen2.5-3B & 0.6975 & \textcolor{red}{0.5848} & 0.0807 & 0.333 & \textbf{0.1425} & 0.3425 \\
 & OMVV       & \textbf{0.7325} & 0.1649 & 0.1919 & \textbf{0.178} &  0.1952 & \textbf{0.3080} \\
\midrule
\multirow{3}{*}{Zebra Logic}
 & Qwen2-1.5B & 0.7500 & \textcolor{red}{0.5562} & 0.1617 & 0.359 & 0.2718 & 0.3718 \\
 & Qwen2.5-3B & 0.7250 & \textcolor{red}{0.8361} & 0.0221 & 0.429 & \textbf{0.1962} & 0.3962  \\
 & OMVV       & \textbf{0.8550} & 0.1844 & 0.1982 & \textbf{0.191} & 0.2461 & \textbf{0.3659} \\
\midrule
\multirow{3}{*}{TriviaQA}
 & Qwen2-1.5B & 0.8025 & \textcolor{red}{0.4031} & 0.0410 & 0.222 &  0.1932 & 0.2932 \\
 & Qwen2.5-3B & 0.8050 & \textcolor{red}{0.3250} & 0.1746 & 0.250 & \textbf{0.1447} & 0.3447 \\
 & OMVV       & \textbf{0.8325} & 0.1987 & 0.1993 & \textbf{0.199} & 0.1650 & \textbf{0.2778} \\
\midrule
\multirow{3}{*}{DeepMind Math}
 & Qwen2-1.5B & 0.9267 & \textcolor{red}{0.4805} & \textcolor{red}{0.4635} & 0.472 &  0.1717 & 0.2717 \\
 & Qwen2.5-3B & 0.9383 & \textcolor{red}{0.2558} & \textcolor{red}{0.5635} & 0.410 &  0.1700 & 0.3700 \\
 & OMVV       & \textbf{0.9667} & 0.1384 & 0.1107 & \textbf{0.125} &  \textbf{0.1487} & \textbf{0.2627} \\
\bottomrule
\end{tabular}
\end{center}
\end{table}

\subsection{RESULTS}
The main goal of OMVV is to improve the accuracy and cost tradeoff over
single verifier selective verification while maintaining the target error rates.
As shown in Table~\ref{tab:pooled-results}, OMVV achieves the highest accuracy
and the lowest average error across all four datasets at
$\alpha_I=\alpha_{II}=0.2$, while also attaining the lowest overall verification
cost. The two single verifier baselines exceed the $\operatorname{Err}_I$ target
on every dataset, whereas OMVV satisfies both target error rates throughout.

A primary reason for OMVV's accuracy and cost advantage is that it is not committed to a single verifier over the entire stream. Its routing policy
updates the probability of selecting each verifier at every round according to the
observed verification loss, allowing the algorithm to shift toward the verifier
that offers a better cost and performance tradeoff for the current portion of the
stream. This adaptive routing is supported by verifier specific score adaptation,
with each verifier maintaining its own score combiner and threshold pair. For
example, on MATH, the weak verifiers assign average raw scores of about
$0.83$--$0.84$ even to incorrect responses, indicating limited separation from
the raw scores alone. Incorporating agreement among the generator's candidate
responses provides an additional signal for the adaptive threshold decisions.

The oracle query frequency should be interpreted together with the cost of the
selected weak verifier. A lower oracle query frequency does not necessarily imply
a lower overall verification cost, since a more expensive single verifier may
require fewer oracle queries because its decisions are less often uncertain, while
still incurring a higher verifier cost on every round. OMVV may therefore query
the oracle more frequently than a single more expensive verifier on some datasets
while still achieving a lower overall verification cost by routing a fraction of
rounds to the cheaper verifier. The improvements in
Table~\ref{tab:pooled-results} thus reflect the combined effect of adaptive
verifier routing, verifier specific score adaptation, and selective oracle use.

Figure~\ref{fig:cheap-verifier-by-difficulty} illustrates how OMVV's routing adapts to changes in problem difficulty. OMVV has no prior knowledge of which
verifier is better suited to a given problem and receives no explicit difficulty label; instead, its routing policy is updated from the oracle feedback collected over time. As harder problems enter the stream, the exponential weights update described in Section~\ref{sec:algorithm} shifts more probability toward the higher cost verifier over the rounds that follow. On MATH, for instance, the probability of routing to the cheaper verifier falls from $0.551$ on easy problems to $0.190$ on hard ones, indicating a substantial shift toward the stronger verifier as the stream becomes more difficult. This behavior reflects the motivation in Section~\ref{sec:psetup}: rather than committing to a single verifier for the entire stream, OMVV adapts its routing policy online as the relative performance of the available verifiers changes.

Table~\ref{tab:minibatch-math} extends this analysis beyond a single difficulty change to a setting where the difficulty of the incoming stream shifts repeatedly, at a rate controlled by the block size $B$: larger values
of $B$ correspond to less frequent shifts, while $B=1$ alternates the difficulty at every problem. Across all block sizes, OMVV achieves the highest accuracy, the lowest average error, and the lowest verification cost, while also satisfying both target error rates. Notably, even under the most rapidly changing stream with $B=1$, OMVV maintains an accuracy of $0.7250$ and an average error of $0.1493$, whereas both single-verifier baselines violate at least one of the target error rates. These results show that OMVV's advantage persists even when the stream difficulty shifts frequently.

\begin{table}[t]
\caption{OMVV versus single-verifier baselines on the MATH
stream, at $\alpha_I = \alpha_{II} = 0.2$, across block sizes $B$. $\text{Err}_I$ and $\text{Err}_{II}$ values exceeding the $0.2$ target are shown in red; the best value in each of Accuracy, Avg($\text{Err}_I$, $\text{Err}_{II}$), Cost, and Oracle
frequency, within each block size, is bolded.}
\label{tab:minibatch-math}
\begin{center}
\begin{tabular}{clccccccc}
\toprule
$B$ & Method & Accuracy & $\text{Err}_I$ & $\text{Err}_{II}$ & Avg & Oracle freq & Cost \\
\midrule
\multirow{3}{*}{100}
 & Qwen2-1.5B & 0.6675 & \textcolor{red}{0.7288} & 0.0396 & 0.3842 & 0.1275 & 0.2275\\
 & Qwen2.5-3B & 0.6975 & \textcolor{red}{0.2000} & \textcolor{red}{0.2116} & 0.2058 & \textbf{0.1129} & 0.3129 \\
 & OMVV       & \textbf{0.7225} & 0.1333 & 0.1581 & \textbf{0.1457} & 0.1035 & \textbf{0.2090} \\
\midrule
\multirow{3}{*}{50}
 & Qwen2-1.5B & 0.6725 & \textcolor{red}{0.7927} & 0.0289 & 0.4108 & 0.1587 & 0.2587 \\
 & Qwen2.5-3B & 0.6900 & \textcolor{red}{0.2287} & 0.1862 & 0.2074 & \textbf{0.1199} & 0.3199 \\
 & OMVV       & \textbf{0.7275} & 0.1119 & 0.1532 & \textbf{0.1326} & 0.1328 & \textbf{0.2432} \\
\midrule
\multirow{3}{*}{25}
 & Qwen2-1.5B & 0.6725 & \textcolor{red}{0.7588} & 0.0563 & 0.4076 & 0.1938 & 0.2938 \\
 & Qwen2.5-3B & 0.6850 & \textcolor{red}{0.3621} & 0.0714 & 0.2167 & \textbf{0.1025} & 0.3025 \\
 & OMVV       & \textbf{0.7250} & 0.1308 & 0.1498 & \textbf{0.1403} & 0.1311 & \textbf{0.2423} \\
\midrule
\multirow{3}{*}{1}
 & Qwen2-1.5B & 0.6825 & \textcolor{red}{0.5749} & \textcolor{red}{0.2202} & 0.3975 & 0.1750 & 0.2750 \\
 & Qwen2.5-3B & 0.6875 & \textcolor{red}{0.3518} & 0.0748 & 0.2133 & \textbf{0.1181} & 0.3181 \\
 & OMVV       & \textbf{0.7250} & 0.1385 & 0.1601 & \textbf{0.1493} & 0.1250  & \textbf{0.2383}\\
\bottomrule
\end{tabular}
\end{center}
\end{table}

\section{CONCLUSION}
We introduced OMVV, an online verification algorithm that adaptively routes responses among weak verifiers with different costs and verification performance. OMVV provides distribution-free finite-time error control and achieves sublinear expected regret relative to the best fixed verifier in hindsight. Experiments across multiple benchmarks and stream conditions show higher accuracy and lower verification cost than single-verifier baselines while maintaining the target
error rates.
single verifier baselines while maintaining the target error rates.

\bibliography{references}
\bibliographystyle{references}

\appendix
\section{Proofs}
\label{sec:proofs}
\subsection{Proof of Theorem~\ref{thm:error-control}}
We prove the two bounds via a common argument. We first control the error rates within the substream of rounds routed to each verifier. We then obtain simultaneous guarantees across all $K$ verifiers by a union bound and combine them through the pooled error definitions.

\begin{lemma}[Per-verifier error control]
\label{lem:per-verifier}
Fix a horizon $T \geq 1$ and a verifier $k$. Suppose the threshold
step size $\eta$ is a fixed positive constant, the exploration
probabilities satisfy $P_A, P_R \in (0,1]$ for every $t \leq T$
with $q_{\min} := \min_{t \leq T}\min\{P_A, P_R\} > 0$, and the
initial thresholds satisfy $\tau_R^{(k),1} \leq \tau_A^{(k),1}$ with
both in $[0,1]$. Then for any $\delta \in (0,1)$, with probability at
least $1-\delta$,
\begin{equation}
\label{eq:lem-per-verifier}
  \mathrm{Err}_I^{(k)}(T) \leq \alpha_I + \Delta\big(N_0^{(k)}(T),
  \delta\big), \qquad \mathrm{Err}_{II}^{(k)}(T) \leq \alpha_{II} +
  \Delta\big(N_1^{(k)}(T), \delta\big),
\end{equation}
where, writing $B := 1 + 2\eta/q_{\min}$,
\[
  \Delta(N,\delta) := \frac{B}{\eta N} + \sqrt{\frac{2\log(4/\delta)}
  {Nq_{\min}}} + \frac{\log(4/\delta)}{3Nq_{\min}}, \qquad \Delta(0,
  \delta) := 0.
\]
\end{lemma}
We prove Lemma~\ref{lem:per-verifier} through the following four claims. The threshold updates in ~\eqref{eq:tau-update} and ~\eqref{eq:tau-update-r} fire only when the oracle-query indicator $O_t$
(Eq.~\eqref{eq:loss-estimate}) equals $1$, with query probability
\[
  q_t := \Pr(O_t = 1 \mid \mathcal{F}_{t-1}, k_t, s_{k_t,t}) \in
  \{P_A,\, 1,\, P_R\}
\]
determined by which region $s_{k_t,t}$ falls into, so $q_t \geq
q_{\min}$. Here $\mathcal{F}_{t-1}$ is the history immediately before
round $t$ (all earlier pairs, decisions, and observed oracle labels,
together with every candidate's raw score and agreement feature
generated alongside $p_t$). For the analysis, let $\mathcal H_t:=\sigma(\mathcal F_{t-1},k_t,s_{k_t,t},g_t)$ denote the information after verifier selection and scoring, with the latent label $g_t$ fixed, but before the oracle-query randomization $O_t$ is drawn; the inclusion of $g_t$ is only for analysis and does not imply that OMVV observes it unless the oracle is queried.

Write $x_t^{A,(k)} := \mathbf{1}[s_{k,t} > \tau_A^{(k),t}] - \alpha_I
\in [-1,1]$, and define
\begin{equation}
\label{eq:e-hat-accept}
  e_t^{A,(k)} := \mathbf{1}[k_t=k]\,\mathbf{1}[g_t=0]\,x_t^{A,(k)},
  \qquad
  \widehat e_t^{A,(k)} :=
  \mathbf{1}[k_t=k]\,\mathbf{1}[g_t=0]\,\frac{O_t}{q_t}\,x_t^{A,(k)}.
\end{equation}
The reject quantities $x_t^{R,(k)}$, $e_t^{R,(k)}$, and
$\widetilde e_t^{R,(k)}$ are defined analogously, replacing
$\alpha_I$ with $\alpha_{II}$,
$\mathbf{1}[s_{k,t}>\tau_A^{(k),t}]$ with
$\mathbf{1}[s_{k,t}\le\tau_R^{(k),t}]$, and
$\mathbf{1}[g_t=0]$ with $\mathbf{1}[g_t=1]$.

\textbf{Claim 1 (Unbiased importance weighting).} For every $t$ and every fixed $k$,
\[
\mathbb{E}\!\left[
\widehat e_t^{A,(k)} \mid \mathcal H_t
\right]
=
e_t^{A,(k)},
\qquad
\mathbb{E}\!\left[
\widehat e_t^{R,(k)} \mid \mathcal H_t
\right]
=
e_t^{R,(k)}.
\]
Consequently,
$Z_t^{A,(k)}:=\widehat e_t^{A,(k)}-e_t^{A,(k)}$ and
$Z_t^{R,(k)}:=\widehat e_t^{R,(k)}-e_t^{R,(k)}$
have conditional mean zero. Hence their partial sums form martingales
with respect to the corresponding analysis filtration.
\begin{proof}
Condition on $\mathcal H_t$. If $k_t\neq k$, then both
$\widehat e_t^{A,(k)}$ and $e_t^{A,(k)}$ are zero. If $k_t=k$,
the quantities $\mathbf{1}[g_t=0]$ and $x_t^{A,(k)}$ are fixed
under the conditioning. The only remaining randomness is
$O_t\sim\mathrm{Bernoulli}(q_t)$, so
\[
\mathbb{E}\!\left[\widehat e_t^{A,(k)}\mid\mathcal H_t\right]
=
\mathbf{1}[k_t=k]\mathbf{1}[g_t=0]x_t^{A,(k)}
\frac{\mathbb{E}[O_t\mid\mathcal H_t]}{q_t}
=
e_t^{A,(k)}.
\]
The reject-side identity follows analogously.
\end{proof}
\textbf{Claim 2 (Telescoping).} For every $T \geq 1$,
\[
  \sum_{t=1}^T \widehat e_t^{A,(k)} \leq \frac{\tau_A^{(k),T+1} -
  \tau_A^{(k),1}}{\eta}, \qquad
  \sum_{t=1}^T \widehat e_t^{R,(k)} \leq \frac{\tau_R^{(k),1} -
  \tau_R^{(k),T+1}}{\eta}.
\]
\begin{proof}
On a round with $k_t = k$ and $O_t = 1$,
\eqref{eq:tau-update} gives $\tau_A^{(k),t+1} \geq
\tau_A^{(k),t} + \eta\, \widehat e_t^{A,(k)}$, since the maximum in
\eqref{eq:tau-update} is at least its second argument; hence
\[
  \widehat e_t^{A,(k)} \leq \frac{\tau_A^{(k),t+1} -
  \tau_A^{(k),t}}{\eta}
\]
in every case. Summing over $t = 1, \dots, T$ telescopes the
right-hand side to $(\tau_A^{(k),T+1} - \tau_A^{(k),1})/\eta$. On the reject side, the bracket in \eqref{eq:tau-update-r} equals $-\widehat e_t^{R,(k)}$ exactly on rounds with $k_t = k$ and $O_t = 1$, since $\widehat e_t^{R,(k)}$ is
built from $x_t^{R,(k)} = \mathbf{1}[s_{k,t} \leq \tau_R^{(k),t}] -
\alpha_{II}$, the negative of the term appearing there. Thus
$\tau_R^{(k),t+1} \leq \tau_R^{(k),t} - \eta\, \widehat e_t^{R,(k)}$
whenever an update fires, giving $\widehat
e_t^{R,(k)} \leq (\tau_R^{(k),t} - \tau_R^{(k),t+1})/\eta$ in every
case; summing telescopes to $(\tau_R^{(k),1} -
\tau_R^{(k),T+1})/\eta$.
\end{proof}
\textbf{Claim 3 (Uniform boundedness of thresholds).} For all $t \leq
T+1$,
\[
  \tau_A^{(k),t},\, \tau_R^{(k),t} \in \left[-\frac{\eta}{q_{\min}},\,
  1 + \frac{\eta}{q_{\min}}\right],
\]
and hence $|\tau_A^{(k),T+1} - \tau_A^{(k),1}| \leq B$ and $|\tau_R^{(k),T+1} - \tau_R^{(k),1}| \leq B$, with $B$ as in Lemma~\ref{lem:per-verifier}.
\begin{proof}
Since $x_t^{A,(k)} \in [-1,1]$ and $O_t/q_t \leq 1/q_{\min}$, every raw increment satisfies $|\eta\, \widehat e_t^{A,(k)}| \leq \eta/q_{\min}$, and likewise on the reject side. We show by induction on $t$ that $\tau_A^{(k),t}, \tau_R^{(k),t}$ both remain in $I := [-\eta/q_{\min},\, 1+\eta/q_{\min}]$ the base case $t=1$ holds since
the initial thresholds lie in $[0,1] \subset I$. Assume both thresholds lie in $I$ at round $t$, and consider the raw, pre-projection accept update in three cases. If $\tau_A^{(k),t} \in
[0,1]$, the update moves the threshold by at most $\eta/q_{\min}$, so
the result lies in $I$. If $\tau_A^{(k),t} > 1$, then $s_{k,t} \leq 1 
\tau_A^{(k),t}$, so $x_t^{A,(k)} = -\alpha_I < 0$ and the raw update
can only decrease the threshold: it stays below $1+\eta/q_{\min}$, and
decreases by at most $\eta/q_{\min}$, so it stays above
$1-\eta/q_{\min} \geq -\eta/q_{\min}$. If $\tau_A^{(k),t} < 0$, then
$s_{k,t} \geq 0 > \tau_A^{(k),t}$, so $x_t^{A,(k)} = 1-\alpha_I > 0$
and the symmetric argument keeps the raw update in $I$. (At the
boundary $\tau_A^{(k),t}=0$ with $s_{k,t}=0$, the indicator
$\mathbf{1}[s_{k,t} > \tau_A^{(k),t}]$ evaluates to $0$, so this
configuration falls under the first case rather than the third.)
Finally, the projection $\max(\tau_R^{(k),t}, \cdot)$ in
~\eqref{eq:tau-update} preserves $I$, since
$\tau_R^{(k),t} \in I$ by the induction hypothesis, completing the
induction for $\tau_A^{(k),t}$.

The argument for $\tau_R^{(k),t}$ is the mirror image, using
$\mathbf{1}[s_{k,t} \leq \tau_R^{(k),t}]$: for $\tau_R^{(k),t} > 1$ the
indicator equals $1$ and $x_t^{R,(k)} = 1 - \alpha_{II}$ makes the
update term $\alpha_{II} - 1 < 0$, decreasing the threshold; for
$\tau_R^{(k),t} < 0$ the indicator equals $0$ and the update term
$\alpha_{II} > 0$ increases it.

Since both endpoints $\tau_A^{(k),1}$ and $\tau_A^{(k),T+1}$ lie in an
interval of width $B = 1+2\eta/q_{\min}$, $|\tau_A^{(k),T+1} -
\tau_A^{(k),1}| \leq B$, and likewise on the reject side.
\end{proof}
\textbf{Claim 4 (Freedman concentration).} Let $M_T^{A,(k)} :=
\sum_{t=1}^T Z_t^{A,(k)}$. With probability at least $1-\delta/2$,
\[
  \big|M_T^{A,(k)}\big| \leq N_0^{(k)}(T)\, \varepsilon\big(N_0^{(k)}(T),
  \delta\big), \qquad \varepsilon(N,\delta) := \sqrt{\frac{2\log(4/\delta)}
  {Nq_{\min}}} + \frac{\log(4/\delta)}{3Nq_{\min}},
\]
and symmetrically $|M_T^{R,(k)}| \leq N_1^{(k)}(T)\,
\varepsilon(N_1^{(k)}(T), \delta)$ with probability at least $1-\delta/2$.
\begin{proof}
Write $Z_t^{A,(k)} = \mathbf{1}[k_t=k]\,\mathbf{1}[g_t=0]\,x_t^{A,(k)}
\big(O_t/q_t - 1\big)$. Since $|x_t^{A,(k)}| \leq 1$ and $q_t \geq
q_{\min}$,
\[
  \big|Z_t^{A,(k)}\big| \leq \left|\frac{O_t}{q_t} - 1\right| \leq
  \max\left\{1, \frac{1}{q_t}-1\right\} \leq \frac{1}{q_t} \leq
  \frac{1}{q_{\min}}.
\]
By Claim~1,
$\mathbb{E}[Z_t^{A,(k)}\mid\mathcal H_t]=0$, so
\[
  \mathrm{Var}\big(Z_t^{A,(k)} \mid H_t= \mathbf{1}[k_t=k]\,\mathbf{1}[g_t=0]\,\big(x_t^{A,(k)}\big)^2
  \cdot \frac{1-q_t}{q_t} \leq \frac{\mathbf{1}[k_t=k]\,\mathbf{1}[g_t=0]}
  {q_{\min}},
\]
using $\mathbb{E}[(O_t/q_t - 1)^2 \mid \mathcal H_t] = (1-q_t)/q_t$ for $O_t
\sim \mathrm{Bernoulli}(q_t)$. Summing over $t \leq T$,
\[
  V_T^{A,(k)} := \sum_{t=1}^T \mathrm{Var}\big(Z_t^{A,(k)} \mid \mathcal H_t\big)
  \leq \frac{N_0^{(k)}(T)}{q_{\min}}.
\]
Applying Freedman's inequality to the martingale $M_T^{A,(k)}$ with
increment bound $1/q_{\min}$ and variance proxy $V_T^{A,(k)} \leq
N_0^{(k)}(T)/q_{\min}$, allocating tail probability $\delta/4$ to each
side, gives $|M_T^{A,(k)}| \leq N_0^{(k)}(T)\, \varepsilon(N_0^{(k)}(T),
\delta)$ with probability at least $1-\delta/2$. The reject-side bound
follows identically.
\end{proof}
\begin{proof}[Proof of Lemma~\ref{lem:per-verifier}]
By Claims 2 and 3,
\[
  \sum_{t=1}^T \widehat e_t^{A,(k)} \leq \frac{|\tau_A^{(k),T+1} -
  \tau_A^{(k),1}|}{\eta} \leq \frac{B}{\eta}, \qquad
  \sum_{t=1}^T \widehat e_t^{R,(k)} \leq \frac{B}{\eta}.
\]
Since $Z_t^{A,(k)} = \widehat e_t^{A,(k)} - e_t^{A,(k)}$,
\[
  \sum_{t=1}^T e_t^{A,(k)} = \sum_{t=1}^T \widehat e_t^{A,(k)} -
  M_T^{A,(k)} \leq \frac{B}{\eta} + \big|M_T^{A,(k)}\big|,
\]
which by Claim 4 is at most $B/\eta + N_0^{(k)}(T)\,
\varepsilon(N_0^{(k)}(T), \delta)$ with probability at least
$1-\delta/2$.

By construction, within verifier $k$'s Accept region, $d_t = A$ if
and only if $s_{k,t} > \tau_A^{(k),t}$, so for $t \in T_k(T)$,
\[
  \mathbf{1}[g_t = 0, d_t = A] = \mathbf{1}[g_t = 0]\,
  \mathbf{1}[s_{k,t} > \tau_A^{(k),t}].
\]
Summing over $t \in T_k(T)$ and using $\sum_{t=1}^T e_t^{A,(k)} =
\sum_{t \in T_k(T)} \mathbf{1}[g_t=0]\mathbf{1}[s_{k,t}>\tau_A^{(k),t}]
- \alpha_I N_0^{(k)}(T)$, this gives the exact identity
\[
  N_0^{(k)}(T)\, \mathrm{Err}_I^{(k)}(T) = \alpha_I N_0^{(k)}(T) +
  \sum_{t=1}^T e_t^{A,(k)} \leq \alpha_I N_0^{(k)}(T) + \frac{B}{\eta}
  + N_0^{(k)}(T)\, \varepsilon\big(N_0^{(k)}(T), \delta\big)
\]
with probability at least $1-\delta/2$. Dividing by $N_0^{(k)}(T)$
gives $\mathrm{Err}_I^{(k)}(T) \leq \alpha_I + \Delta(N_0^{(k)}(T),
\delta)$. The identical argument on the reject side holds with
probability at least $1-\delta/2$; a union bound over the two events
gives both simultaneously with probability at least $1-\delta$,
proving the lemma.
\end{proof}
\begin{proof}[Proof of Theorem~\ref{thm:error-control}]
Apply Lemma~\ref{lem:per-verifier} to each verifier $k = 1, \dots, K$
at confidence level $\delta/K$; by a union bound over $k$, all $K$
per-verifier bounds hold simultaneously with probability at least
$1-\delta$. Since every round $t \leq T$ is routed to exactly one
verifier, $N_0(T) = \sum_{k=1}^K N_0^{(k)}(T)$ and
\[
  N_0(T)\, \mathrm{Err}_I(T) = \sum_{k=1}^K N_0^{(k)}(T)\,
  \mathrm{Err}_I^{(k)}(T).
\]
Substituting the bound of Lemma~\ref{lem:per-verifier} for each
$\mathrm{Err}_I^{(k)}(T)$, at confidence $\delta/K$, and dividing by
$N_0(T)$ gives
\[
  \mathrm{Err}_I(T) \leq \alpha_I + \underbrace{\frac{1}{N_0(T)}
  \sum_{k=1}^K N_0^{(k)}(T)\, \Delta\big(N_0^{(k)}(T),
  \delta/K\big)}_{=: \, \Delta^I_{\mathrm{pool}}(T,\delta)}
\]
with probability at least $1-\delta$. The false-reject bound on
$\mathrm{Err}_{II}(T)$ follows identically, defining
$\Delta^{II}_{\mathrm{pool}}(T,\delta)$ analogously with $N_1^{(k)}(T)$
in place of $N_0^{(k)}(T)$.
\end{proof}
\subsection{Proof of Theorem~\ref{thm:regret}}
We bound the regret through three steps: showing $\hat\ell_t(k)$ is
an unbiased estimate of $\ell_t(k)$ for every verifier, a
deterministic inequality governing the Hedge weight update, and a
bound on the estimate's second moment; combining the three yields the
stated rate.

For the regret analysis, we regard each verifier's score $s_{k,t}$ and oracle query probability $q_t(k)$ as well defined at every round, conditional on the history and the current candidate context, regardless of whether verifier $k$ is selected. These quantities are used only in the analysis. During online execution, OMVV evaluates and observes only the score of the selected verifier $k_t$.

\textbf{Claim 1 (Unbiased loss estimate).} Let $
\mathcal G_t
:=
\sigma\!\left(
\mathcal F_{t-1},
g_t,
\{s_{j,t},q_t(j)\}_{j=1}^K
\right)
$. For every round $t$ and every $k \in \{1,\dots,K\}$, selected or not,
\[
  \mathbb{E}\big[\hat\ell_t(k) \mid \mathcal{G}_t\big] = \ell_t(k),
  \qquad
  \mathbb{E}\big[\langle \bar w^t, \hat\ell_t \rangle \mid
  \mathcal{G}_t\big] = \mathbb{E}\big[\ell_t(k_t) \mid
  \mathcal{G}_t\big].
\]
\begin{proof}
All denominators in ~\eqref{eq:loss-estimate} are strictly positive: $\bar w_k^t > 0$ and $q_t(k) \geq q_{\min} > 0$. Condition on $\mathcal{G}_t$, under which $\bar w^t$, every $s_{k,t}$,
every $q_t(k)$, $c_k$, and $g_t$ are fixed. Chaining the routing draw
and the query indicator,
\[
  \mathbb{E}\big[\mathbf{1}[k_t=k]\, O_t \mid \mathcal{G}_t\big] =
  \bar w_k^t\, q_t(k),
\]
so the importance-weighted term in ~\eqref{eq:loss-estimate}
has conditional expectation $1$, giving $\mathbb{E}[\hat\ell_t(k) \mid
\mathcal{G}_t] = \lambda(s_{k,t}-g_t)^2 + (1-\lambda)c_k = \ell_t(k)$. For the second identity, $\bar w^t$ is $\mathcal{G}_t$-measurable, so
$\mathbb{E}[\langle \bar w^t, \hat\ell_t\rangle \mid \mathcal{G}_t] =
\sum_k \bar w_k^t\, \ell_t(k) = \mathbb{E}[\ell_t(k_t) \mid
\mathcal{G}_t]$.
\end{proof}
\textbf{Claim 2 (Potential-function inequality).} Suppose the routing
weights are initialized uniformly, $w_k^1 = 1/K$ for every $k$. Then
for every fixed comparator $k^\star$ and every $\eta_H > 0$,
\[
  \sum_{t=1}^T \big\langle \bar w^t, \hat\ell_t \big\rangle -
  \sum_{t=1}^T \hat\ell_t(k^\star) \leq \frac{\log K}{\eta_H} +
  \frac{\eta_H}{2} \sum_{t=1}^T \sum_{k=1}^K \bar w_k^t\,
  \hat\ell_t(k)^2.
\]
\begin{proof}
Let $\Phi_t := \log\big(\sum_k w_k^t\big)$, so $\Phi_1 = \log(K \cdot
1/K) = 0$. By ~\eqref{eq:weight-update},
\[
  \Phi_{t+1} - \Phi_t = \log\left(\sum_k \bar w_k^t\,
  e^{-\eta_H \hat\ell_t(k)}\right).
\]
Since $\hat\ell_t(k) \geq 0$ for every $k$ and $t$, the inequality
$e^{-x} \leq 1 - x + x^2/2$ holds with $x = \eta_H \hat\ell_t(k) \geq
0$. Using $\sum_k \bar w_k^t = 1$ and then $\log(1+u) \leq u$,
\[
  \Phi_{t+1} - \Phi_t \leq -\eta_H \big\langle \bar w^t, \hat\ell_t
  \big\rangle + \frac{\eta_H^2}{2} \sum_k \bar w_k^t\,
  \hat\ell_t(k)^2.
\]
Summing over $t = 1, \dots, T$ and using $\Phi_1 = 0$ bounds
$\Phi_{T+1}$ from above.

For a lower bound, positivity of the weights gives
\[
  \Phi_{T+1} \geq \log w_{k^\star}^{T+1} = \log\left(\frac{1}{K}
  \exp\left(-\eta_H \sum_{t=1}^T \hat\ell_t(k^\star)\right)\right) =
  -\log K - \eta_H \sum_{t=1}^T \hat\ell_t(k^\star).
\]
Combining the two bounds, rearranging, and dividing by $\eta_H > 0$
gives the claim.
\end{proof}
\textbf{Claim 3 (Second-moment bound).} Write $b_{k,t} :=
(s_{k,t}-g_t)^2 \in [0,1]$. For every round $t$,
\[
  \sum_{k=1}^K \bar w_k^t\, \mathbb{E}\big[\hat\ell_t(k)^2 \mid
  \mathcal{G}_t\big] \leq \lambda^2 \sum_{k=1}^K \frac{b_{k,t}^2}
  {q_t(k)} + \big(1-\lambda^2\big) \leq \lambda^2 \sum_{k=1}^K
  \frac{1}{q_t(k)} + \big(1-\lambda^2\big).
\]

\begin{proof}
Write $X_t(k) := \mathbf{1}[k_t=k]\, O_t / \big(\bar w_k^t\,
q_t(k)\big)$, so ~\eqref{eq:loss-estimate} reads $\hat\ell_t(k)
= \lambda\, b_{k,t}\, X_t(k) + (1-\lambda) c_k$. Since $\mathbf{1}[k_t=k]$
and $O_t$ are $\{0,1\}$-valued, squaring leaves the numerator
unchanged, so
\[
  \mathbb{E}\big[X_t(k)^2 \mid \mathcal{G}_t\big] = \frac{1}{\bar
  w_k^t\, q_t(k)}, \qquad \mathbb{E}\big[X_t(k) \mid \mathcal{G}_t\big]
  = 1,
\]
the second identity following as in Claim 1's proof. Condition on $\mathcal{G}_t$, under which $b_{k,t}$, $c_k$, and
$\lambda$ are fixed. Expanding the square,
\[
  \mathbb{E}\big[\hat\ell_t(k)^2 \mid \mathcal{G}_t\big] = \lambda^2\,
  \frac{b_{k,t}^2}{\bar w_k^t\, q_t(k)} + 2\lambda(1-\lambda)\,
  b_{k,t}\, c_k + (1-\lambda)^2 c_k^2.
\]
Multiplying by $\bar w_k^t$ and summing over $k$, the factor $\bar
w_k^t$ cancels the $1/\bar w_k^t$ in the first term exactly, leaving
\[
  \sum_{k=1}^K \bar w_k^t\, \mathbb{E}\big[\hat\ell_t(k)^2 \mid
  \mathcal{G}_t\big] = \lambda^2 \sum_{k=1}^K \frac{b_{k,t}^2}{q_t(k)}
  + \sum_{k=1}^K \bar w_k^t \Big[2\lambda(1-\lambda)\, b_{k,t}\, c_k +
  (1-\lambda)^2 c_k^2\Big].
\]
Since $b_{k,t} \leq 1$ and $c_k \leq 1$, the bracketed quantity is at
most $2\lambda(1-\lambda) + (1-\lambda)^2 = (1-\lambda)(1+\lambda) =
1-\lambda^2$, and $\bar w_k^t$ averages to at most this same bound.
The second inequality in the claim follows from $b_{k,t} \leq 1$.
\end{proof}
Claim 1 gives $\mathbb{E}[\langle \bar w^t, \hat\ell_t\rangle \mid
\mathcal{G}_t] = \mathbb{E}[\ell_t(k_t) \mid \mathcal{G}_t]$ and
$\mathbb{E}[\hat\ell_t(k^\star) \mid \mathcal{G}_t] =
\ell_t(k^\star)$. Taking expectations of both and summing over $t$
turns the left-hand side of Claim 2 into $\mathbb{E}[R_T(k^\star)]$.

For the right-hand side, taking expectations and applying Claim 3,
then using $q_t(k) \geq q_{\min}$ for every $k, t$,
\[
  \mathbb{E}\left[\sum_{t=1}^T \sum_{k=1}^K \bar w_k^t\,
  \hat\ell_t(k)^2\right] \leq \sum_{t=1}^T \left(\lambda^2
  \sum_{k=1}^K \frac{1}{q_t(k)} + (1-\lambda^2)\right) \leq
  \lambda^2 \frac{KT}{q_{\min}} + (1-\lambda^2) T.
\]

Combining,
\[
  \mathbb{E}\big[R_T(k^\star)\big] \leq \frac{\log K}{\eta_H} +
  \frac{\eta_H}{2}\left(\lambda^2 \frac{KT}{q_{\min}} +
  (1-\lambda^2)T\right)
\]
for every $\eta_H > 0$. The right-hand side has the form $A/\eta_H +
\eta_H B/2$ with $A = \log K$ and $B = \lambda^2 KT/q_{\min} +
(1-\lambda^2)T$, minimized at $\eta_H^\star = \sqrt{2A/B}$ with value
$\sqrt{2AB}$, giving
\[
  \mathbb{E}\big[R_T(k^\star)\big] \leq \sqrt{2\log K
  \left(\lambda^2 \frac{KT}{q_{\min}} + (1-\lambda^2)T\right)} =
  O\!\left(\sqrt{T\log K}\cdot\sqrt{\frac{\lambda^2 K}{q_{\min}} +
  (1-\lambda^2)}\right),
\]
which is sublinear in $T$.
\section{ADDITIONAL EXPERIMENTS}
\label{sec:addexp}
\subsection{Empirical Validation of Sublinear Regret}
\label{sec:regret-empirical}
Theorem~\ref{thm:regret} establishes that OMVV's expected regret against
the best fixed verifier in hindsight, $\mathbb{E}[R_T(k^\star)]$, grows
sublinearly in $T$. To validate this empirically, we compute the
realized (non-expected) trajectory $R_t(k^\star) = \sum_{s=1}^{t}
\big(\ell_s(k_s) - \ell_s(k^\star)\big)$ on several datasets
(MATH, DeepMind Math \texttt{linear\_2d}, and TriviaQA), using the same $K{=}2$ verifier pool (\texttt{qwen2\_1.5b}, \texttt{qwen2.5\_3b}) as the main reported results. Here $k_s$ is the verifier OMVV actually selected at round $s$, $k^\star = \arg\min_k \sum_{t=1}^T \ell_t(k)$ is fixed once from the full-horizon totals, and $\ell_t(k)$ is computed from each verifier's own combiner-transformed score at every round.  We compute the
same quantity for the single-verifier SSV baseline, whose verifier choice is
fixed throughout the stream.

Figure~\ref{fig:regret-empirical} compares the cumulative regret trajectories of
OMVV and SSV across the three datasets. In all cases, OMVV's cumulative regret
rises steeply early in the stream and then flattens over the remainder, indicating
that OMVV's per-round loss has essentially matched that of the best fixed verifier.
This decelerating growth is consistent with the sublinear expected regret guarantee
in Theorem~\ref{thm:regret}. In contrast, the SSV baseline exhibits approximately
linear cumulative regret growth over the observed horizon, with no comparable
flattening. The comparison illustrates the cost of committing to a single verifier
when the better one for a given stream is not known in advance.
\begin{figure}[t]
\centering
\includegraphics[width=\textwidth]{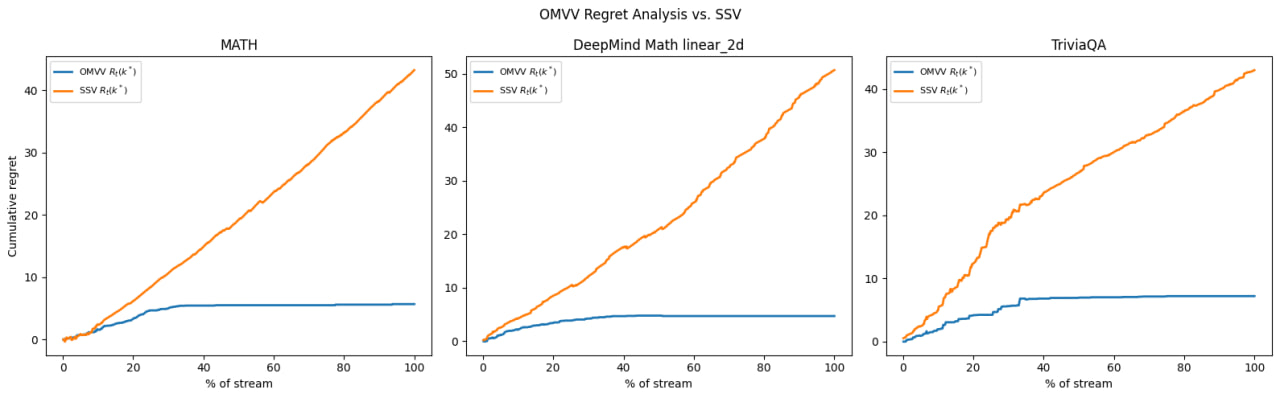}
\caption{Empirical cumulative regret $R_t(k^\star)$ of OMVV and the SSV baseline
relative to the best fixed verifier in hindsight. OMVV exhibits flattening
cumulative regret, consistent with the sublinear expected regret guarantee of
Theorem~\ref{thm:regret}. In contrast, SSV shows approximately linear regret
growth over the observed horizon.}
\label{fig:regret-empirical}
\end{figure}
\subsection{Block size Experiment: DeepMind Math (\texttt{linear\_2d})}
\label{sec:linear2d-minibatch}
We conduct an additional experiment on DeepMind Math
(\texttt{linear\_2d}) using the same protocol as
Table~\ref{tab:minibatch-math}, with the block size $B$ controlling how
frequently the stream alternates between easy and hard problems.
Table~\ref{tab:linear2d-minibatch} reports the results. Across all block sizes, OMVV achieves the highest accuracy, the lowest average error, and the lowest verification cost among the three methods. It also achieves the lowest $\operatorname{Err}_I$ at every block size. Neither single verifier baseline satisfies both error targets for any value of $B$, whereas OMVV satisfies both targets throughout. These results show that OMVV remains effective under repeated changes in problem difficulty on DeepMind Mathematics as well as MATH.
\begin{table}[t]
\centering
\caption{OMVV versus single-verifier baselines on the DeepMind Math
non-stationary stream, at $\alpha_I = \alpha_{II} = 0.2$, across block
sizes $B$. $\text{Err}_I$ and $\text{Err}_{II}$ values exceeding the
$0.2$ target are shown in red; the best value in each of Accuracy,
Avg($\text{Err}_I$, $\text{Err}_{II}$), Cost, and Oracle frequency, within
each block size, is bolded.}
\label{tab:linear2d-minibatch}
\begin{tabular}{clcccccc}
\toprule
$B$ & Method & Accuracy & Err$_I$ & Err$_{II}$ & Avg & Cost & Oracle freq \\
\midrule
\multirow{3}{*}{100}
& qwen2\_1.5b & 0.9333 & \textcolor{red}{0.2288} & \textcolor{red}{0.6786} & 0.4537 & 0.2903 & 0.1903 \\
& qwen2.5\_3b & 0.9267 & \textcolor{red}{0.8627} & \textbf{0.0314} & 0.4471 & 0.2833 & \textbf{0.0833} \\
& OMVV        & \textbf{0.9433} & \textbf{0.1405} & 0.1417 & \textbf{0.1411} & \textbf{0.2214} & 0.1153 \\
\midrule
\multirow{3}{*}{50}
& qwen2\_1.5b & \textbf{0.9433} & \textcolor{red}{0.2828} & \textcolor{red}{0.6419} & 0.4624 & 0.2836 & 0.1836 \\
& qwen2.5\_3b & 0.9267 & \textcolor{red}{0.8627} & \textbf{0.0365} & 0.4496 & 0.2813 & \textbf{0.0813} \\
& OMVV        & 0.9483 & \textbf{0.1186} & 0.1588 & \textbf{0.1387} & \textbf{0.2343} & 0.1242 \\
\midrule
\multirow{3}{*}{25}
& qwen2\_1.5b & 0.9367 & 0.1870 & \textcolor{red}{0.6791} & 0.4331 & 0.2909 & 0.1909 \\
& qwen2.5\_3b & 0.9233 & \textcolor{red}{0.8679} & \textbf{0.0399} & 0.4539 & 0.3017 & \textbf{0.1017} \\
& OMVV        & \textbf{0.9533} & \textbf{0.1667} & 0.1095 & \textbf{0.1381} & \textbf{0.2486} & 0.1393 \\
\midrule
\multirow{3}{*}{1}
& qwen2\_1.5b & 0.9317 & \textcolor{red}{0.3978} & \textcolor{red}{0.5075} & 0.4527 & 0.2852 & 0.1852 \\
& qwen2.5\_3b & 0.9267 & \textcolor{red}{0.8627} & \textbf{0.0348} & 0.4488 & 0.2847 & \textbf{0.0847} \\
& OMVV        & \textbf{0.9583} & \textbf{0.0991} & 0.1915 & \textbf{0.1453} & \textbf{0.2430} & 0.1384 \\
\bottomrule
\end{tabular}
\end{table}
\subsection{Evaluation at a Different Error-Rate Target}
\label{sec:alpha025-pooled}

We also evaluate OMVV under a different error target by setting
$\alpha_I=\alpha_{II}=0.25$, instead of $0.2$, across all four datasets.
Table~\ref{tab:alpha025-pooled} reports the results. OMVV achieves the
highest accuracy, the lowest verification cost, and the lowest
$\operatorname{Err}_I$ on every dataset. Moreover, OMVV satisfies both
error targets across all four datasets, whereas neither single verifier
baseline satisfies both targets on any dataset. These results indicate that
the empirical advantages of OMVV are not specific to the
$\alpha_I=\alpha_{II}=0.2$ setting used in the main experiments.

\begin{table}[t]
\centering
\caption{OMVV versus single-verifier baselines, pooled easy and hard
subsets, at $\alpha_I = \alpha_{II} = 0.25$. $\text{Err}_I$ and
$\text{Err}_{II}$ values exceeding the $0.25$ target are shown in
red; the best value in each of Accuracy, Avg($\text{Err}_I$,
$\text{Err}_{II}$), Cost, and Oracle frequency, within each dataset, is
bolded.}
\label{tab:alpha025-pooled}
\begin{tabular}{clcccccc}
\toprule
Dataset & Method & Accuracy & Err$_I$ & Err$_{II}$ & Avg & Cost & Oracle freq \\
\midrule
\multirow{3}{*}{MATH}
& qwen2\_1.5b & 0.6775 & \textcolor{red}{0.7845} & 0.0352 & 0.4099 & 0.3367 & 0.2367 \\
& qwen2.5\_3b & 0.7000 & \textcolor{red}{0.4928} & 0.0845 & 0.2887 & 0.3794 & \textbf{0.1794} \\
& OMVV        & \textbf{0.7175} & 0.2018 & 0.2290 & \textbf{0.2154} & \textbf{0.3346} & 0.2114 \\
\midrule
\multirow{3}{*}{zebra\_logic}
& qwen2\_1.5b & 0.7200 & \textcolor{red}{0.6519} & 0.0880 & 0.3700 & 0.3500 & 0.2500 \\
& qwen2.5\_3b & 0.6900 & \textcolor{red}{0.8810} & 0.0243 & 0.4527 & 0.3514 & \textbf{0.1514} \\
& OMVV        & \textbf{0.8300} & 0.2447 & 0.2277 & \textbf{0.2362} & \textbf{0.3430} & 0.2214 \\
\midrule
\multirow{3}{*}{TriviaQA}
& qwen2\_1.5b & 0.7925 & \textcolor{red}{0.4069} & 0.0199 & 0.2134 & 0.3001 & 0.2001 \\
& qwen2.5\_3b & 0.7950 & \textcolor{red}{0.3306} & 0.1897 & 0.2602 & 0.3435 & \textbf{0.1435} \\
& OMVV        & \textbf{0.8325} & 0.1698 & 0.1796 & \textbf{0.1747} & \textbf{0.2882} & 0.1755 \\
\midrule
\multirow{3}{*}{DeepMind Math}
& qwen2\_1.5b & 0.9250 & \textcolor{red}{0.2969} & \textcolor{red}{0.6573} & 0.4771 & 0.2704 & 0.1704 \\
& qwen2.5\_3b & 0.9267 & \textcolor{red}{0.7552} & 0.1195 & 0.4374 & 0.2919 & \textbf{0.0919} \\
& OMVV        & \textbf{0.9583} & 0.2488 & 0.0889 & \textbf{0.1689} & \textbf{0.2651} & 0.1473 \\
\bottomrule
\end{tabular}
\end{table}
\subsection{Verifier Implementation Details and Prompt Specifications}
\label{app:prompts}

Across all four datasets, we use a multi-model pipeline to generate, score,
and verify candidate solutions. \texttt{gpt-4o-mini} serves as the solution
generator, while Qwen2-1.5B-Instruct and Qwen2.5-3B-Instruct serve as the
two weak verifiers. Prompt wording varies across datasets: MATH and DeepMind
Math (\texttt{linear\_2d}) use the same weak verifier prompt, while Zebra
Logic and TriviaQA use prompts tailored to their respective tasks. For MATH
and TriviaQA, correctness is judged by \texttt{gpt-4o} acting as the oracle.
For Zebra Logic and DeepMind Math, correctness is determined by a
deterministic programmatic grader, since both tasks admit an exact check
against the ground truth. We provide representative prompts below for the
three roles: generator, weak verifier, and oracle.
\subsubsection{Generator}
\begin{promptbox}{Generator Prompt --- Math (GPT-4o-mini)}

Solve this math problem step by step. Show your reasoning, then give the final answer.

\medskip
\textbf{PROBLEM:}\\
\{problem\}

\medskip
\textbf{Respond with JSON:}

\medskip
\texttt{\{}\\
\hspace*{1em}\texttt{"reasoning": "Your step-by-step reasoning here...",}\\
\hspace*{1em}\texttt{"final\_answer": "Your final answer (just the answer,}\\
\hspace*{1em}\texttt{e.g., '42' or '3/4' or 'x=5')"}\\
\texttt{\}}

\end{promptbox}
\begin{promptbox}{Generator Prompt --- Zebra Logic (GPT-4o-mini)}

Solve this logic grid puzzle. Each category has one value per column
(entity); use the clues to deduce the unique correct arrangement.

\medskip
\textbf{PUZZLE:}\\
\{question\}

\medskip
\textbf{Respond with JSON:}

\medskip
\texttt{\{}\\
\hspace*{1em}\texttt{"reasoning": "Brief step-by-step deduction...",}\\
\hspace*{1em}\texttt{"answer\_table": "The completed table as plain text,}\\
\hspace*{1em}\texttt{one line per category, in the exact format:}\\
\hspace*{1em}\texttt{| CategoryName | value1 | value2 | ... |}\\
\hspace*{1em}\texttt{(same column order as the puzzle's blank table, left to right)"}\\
\texttt{\}}
\end{promptbox}

\subsubsection{Weak Verifier}

\begin{promptbox}{Weak Verifier Prompt --- MATH}

Rate the correctness of this math solution on a scale from 0.0 to 1.0.

\medskip
\textbf{PROBLEM:}\\
\{problem\}

\medskip
\textbf{REASONING:}\\
\{reasoning\}

\medskip
\textbf{FINAL ANSWER:}\\
\{final\_answer\}

\medskip
\textbf{Consider:}

\begin{enumerate}
    \item Is the reasoning mathematically sound?
    \item Are all calculations correct?
    \item Does the final answer follow from the reasoning?
\end{enumerate}

\textbf{Respond with ONLY a JSON object, nothing else:}

\medskip
\texttt{\{"score": 0.XX\}}

\end{promptbox}

This prompt is reused verbatim for DeepMind Math.

\begin{promptbox}{Weak Verifier Prompt TriviaQA}
Rate the correctness of this trivia answer on a scale from 0.0 to 1.0.

\medskip
\textbf{QUESTION:}\\
\{question\}

\medskip
\textbf{REASONING:}\\
\{reasoning\}

\medskip
\textbf{ANSWER:}\\
\{final\_answer\}

\medskip
\textbf{Consider:}

\begin{enumerate}
    \item Does the reasoning show genuine knowledge of the relevant facts, or is it vague/guessing?
    \item Is the answer specific and directly responsive to the question?
    \item Does the final answer follow logically from the reasoning?
\end{enumerate}

\textbf{Respond with ONLY a JSON object, nothing else:}

\medskip
\texttt{\{"score": 0.XX\}}
\end{promptbox}

\subsubsection{Oracle}

\begin{promptbox}{Oracle Prompt --- MATH (GPT-4o)}
Determine if two math answers are equivalent.

\medskip
\textbf{PROBLEM:}\\
\{problem\}

\medskip
\textbf{STUDENT'S REASONING:}\\
\{reasoning\}

\medskip
\textbf{STUDENT'S ANSWER:}\\
\{generated\_answer\}

\medskip
\textbf{CORRECT ANSWER:}\\
\{ground\_truth\_answer\}

\medskip
\textbf{Are these answers equivalent? Consider:}

\begin{itemize}
    \item Different formats: "5", "5.0", "5/1", "\$5\$", "\textbackslash boxed\{5\}" are all equivalent
    \item Simplified vs unsimplified: "2/4" = "1/2"
    \item Different notation: "x=3" vs "3" (if asking for x)
\end{itemize}

\textbf{Respond with JSON:}

\medskip
\texttt{\{"is\_correct": true/false,}\\
\hspace*{1em}\texttt{"explanation": "brief reason"\}}
\end{promptbox}

\end{document}